\documentclass[letterpaper]{article} 
\usepackage{aaai24}  
\usepackage{times}  
\usepackage{helvet}  
\usepackage{courier}  
\usepackage[hyphens]{url}  
\usepackage{graphicx} 
\usepackage{natbib}  
\usepackage{caption} 
\usepackage{algorithm}
\usepackage{algorithmic}

\usepackage{caption}
\newcommand{\ie}{i.\,e., }
\usepackage{amsmath}
\usepackage{amssymb}
\usepackage{mathtools}
\usepackage{amsthm}
\usepackage{booktabs}
\usepackage{multirow}

\newcommand{\eqnref}[1]{(\ref{#1})}
\newcommand\numberthis{\addtocounter{equation}{1}\tag{\theequation}}

\theoremstyle{plain}
\newtheorem{theorem}{Theorem}[]
\newtheorem{proposition}[theorem]{Proposition}

\theoremstyle{definition}
\newtheorem{definition}[theorem]{Definition}

\theoremstyle{remark}

\newtheorem*{theorem*}{Theorem}

\usepackage{newfloat}
\usepackage{listings}
\DeclareCaptionStyle{ruled}{labelfont=normalfont,labelsep=colon,strut=off} 
\floatstyle{ruled}
\newfloat{listing}{tb}{lst}{}
\floatname{listing}{Listing}
\title{A Simple and Yet Fairly Effective Defense for Graph Neural Networks}
\author{
    Sofiane Ennadir\textsuperscript{\rm 1},
    Yassine Abbahaddou\textsuperscript{\rm 2},
    Johannes F. Lutzeyer\textsuperscript{\rm 2}, \\
    Michalis Vazirgiannis\textsuperscript{\rm 1, 2},
    Henrik Boström\textsuperscript{\rm 1}
}
\affiliations{
    \textsuperscript{\rm 1}EECS, KTH Royal Institute of Technology, Stockholm, Sweden \\
    \textsuperscript{\rm 2}DaSciM, LIX, Ecole Polytechnique, Institut Polytechnique de Paris, France \\
    \{ennadir, mvaz, bostromh\} @ kth.se, \{yassine.abbahaddou, johannes.lutzeyer\} @ polytechnique.edu.
}

\usepackage{bibentry}

\begin{document}

\maketitle

\begin{abstract}
Graph Neural Networks (GNNs) have emerged as the dominant approach for machine learning on graph-structured data. However, concerns have arisen regarding the vulnerability of GNNs to small adversarial perturbations. Existing defense methods against such perturbations suffer from high time complexity and can negatively impact the model's performance on clean graphs. To address these challenges, this paper introduces NoisyGNNs, a novel defense method that incorporates noise into the underlying model's architecture. We establish a theoretical connection between noise injection and the enhancement of GNN robustness, highlighting the effectiveness of our approach. We further conduct extensive empirical evaluations on the node classification task to validate our theoretical findings, focusing on two popular GNNs: the GCN and GIN. The results demonstrate that NoisyGNN achieves superior or comparable defense performance to existing methods while minimizing added time complexity. The NoisyGNN approach is model-agnostic, allowing it to be integrated with different GNN architectures. Successful combinations of our NoisyGNN approach with existing defense techniques demonstrate even further improved adversarial defense results. Our code is publicly available at: https://github.com/Sennadir/NoisyGNN.

\end{abstract}

\section{Introduction}
\label{sec:introduction}

Graphs have garnered substantial recognition in recent years as a powerful approach to represent intricate and irregular data, drawing significant attention across various fields. Their applications span diverse domains, including modeling social networks to discern connections between individuals and representing atom interactions in molecules. Consequently, the proliferation of graph-based applications has necessitated the development of machine learning algorithms tailored for graph processing. One notable and powerful technique is the Graph Neural Network (GNN), which has emerged as a valuable tool for learning node and graph representations. GNNs often belong to the family of Message Passing Neural Networks (MPNNs) \cite{gilmer2017}, such as the Graph Isomorphism Networks (GIN) \cite{xu2019powerful} and the Graph Convolutional Networks (GCN) \cite{Kipf:2017tc}. These GNNs have demonstrated remarkable success in tackling a wide range of real-world problems. For instance, within the field of chemistry, significant attention has been dedicated to employing deep learning systems based on graphs for tasks like drug screening and design, where molecules are effectively represented as graphs \cite{kearnes2016molecular}. GNNs have also shown efficacy in predicting protein functions \cite{you_deep_graph_go} and other biomedical proprieties such as antibiotic resistance \cite{qabel2022structure} since proteins can be effectively modeled as graphs. In a different context, GNNs have found utility in session-based recommendation systems~\cite{wu2019session}.

Given the growing popularity of these techniques, there arises a compelling necessity for an in-depth examination of their robustness to ensure their safe usage in critical sectors such as healthcare. Recent work \cite{dai2018, zugner2018adversarial, gunnemann2022graph} have indeed detected the vulnerability of Graph Neural Networks (GNNs) to adversarial attacks, which are injected intentional input alterations that can manipulate the model’s predictions through slight structural modifications or node feature-based perturbations. To mitigate these effects and bolster the robustness of Message Passing Neural Networks (MPNNs), various studies have proposed diverse techniques. These approaches encompass a spectrum of strategies, including augmenting training data with adversarial examples and subsequent model retraining \cite{graph_adversarial_training}, using edge pruning techniques \cite{gnn_guard}, and the introduction of robustness certificates \cite{schuchardt2021collective}.

While some of these defense methods have exhibited success in countering adversarial perturbations, they often entail a high level of complexity due to their underlying architecture. The time complexity of these methods tends to increase heavily with the graph's size, consequently limiting their practical applicability in certain settings. Furthermore, a significant drawback of many existing approaches is their requirement for extensive architectural modifications, posing challenges for integration into different models. Moreover, some available defense methods can have detrimental effects on the model's performance when applied to clean, i.e., not attacked, graphs. This is of particular importance since in practice, the users do not have prior knowledge of whether their graph datasets have been attacked.

Recently, there has been a notable surge of interest in leveraging adversarial weight perturbation as a means to enhance the generalization capabilities of GNNs \cite{wu2023adversarial}. While existing research has primarily focused on demonstrating its efficacy in improving accuracy, our study considers a new perspective by investigating its potential application in the field of adversarial defense. In particular, we explore a defense strategy, called NoisyGNN, which leverages randomization by introducing random noise in the hidden states of certain layers of the GNN during the training phase.  We start by adapting the mathematical formulation to investigate the robustness of GNNs against graph adversarial attacks. We afterward analyze the impact of randomization in enhancing the robustness of GNNs. Our theoretical analysis establishes a connection between noise injection in the architecture and strengthening a GNN's robustness. Finally, we empirically evaluate the proposed perturbation defense for its effectiveness against various adversarial attack methods in comparison to other available defense approaches on commonly used real-world benchmark datasets. While our theoretical and experimental analyses primarily focus on two well-known GNN architectures, namely GCN and GIN, our approach is model-agnostic and can be easily applied to different network architectures. We can summarize our contribution in the following points:
\begin{itemize}
    \item We provide a mathematical formalization of graph adversarial attacks on GNNs to connect the noise injection's effect to robustness. 
    \item We derive an upper bound, based on a theoretical analysis, that demonstrates the link between randomization and robustness and hence proves the effectiveness of our proposed framework, NoisyGNN, in enhancing the robustness of GCN and GIN based classifiers.
    \item We conduct extensive evaluations of our theoretical findings on the node classification task using various benchmark datasets. Our model is compared to several state-of-the-art defense methods, and in the majority of cases, our proposed framework demonstrates superior or comparable performance while minimizing the added time complexity.    
\end{itemize}

\section{Related Work}
\label{sec:related_work}

Following their success in diverse tasks, questions about the robustness of Deep Learning models have surfaced notably in computer vision \cite{goodfellow_paper, REN2020346}, extending recently to discrete domains such as Natural Language Processing and graphs \cite{gunnemann2022graph}. Adversarial attacks can be categorized into various subgroups depending on the attacker's knowledge and objectives, such as poisoning/evasion and targeted/un-targeted attacks. Notably, Nettack \cite{zugner2018adversarial} introduced a targeted attack method that perturbs both the graph structure and node features. This approach utilizes a greedy optimization algorithm to minimize an attack loss against a surrogate model. Building upon this work, Mettack \cite{zugner2019adversarial} formulates the problem as a bi-level optimization task and employs meta-gradients to tackle it. Expanding on these advancements, \cite{zhan2021black} proposed a black-box gradient attack algorithm that overcomes several limitations of the original techniques. On an alternative front, \cite{dai2018} proposed using Reinforcement Learning to find appropriate graph adversarial attacks.

In addition to the advancements in adversarial attack techniques, research on defense methods against these attacks has gained attention, although it remains relatively less explored compared to computer vision defense strategies. Similar to image-based models, robust training \cite{robust_training_zugner} and aggregation \cite{robust_aggregation} have been proposed as mechanisms to enhance the robustness of GNNs by iteratively augmenting the training set. Furthermore, defense strategies leveraging low-rank matrix approximation combined with graph anomaly detection \cite{Ma_2021} have been employed. For instance, GNN-Jaccard \cite{gnn_jaccard} performs pre-processing on the graph's adjacency matrix to identify potential edge manipulations, while GNN-SVD \cite{gnn_svd} employs a low-rank approximation of the adjacency matrix to filter out noise. Additionally, techniques such as edge pruning \cite{gnn_guard} and transfer learning \cite{Tang_2020} have been utilized to mitigate the impact of poisoning attacks. From another perspective, \citet{seddik_rgcn} add node feature kernels to the GCN's message passing operator to enhance the GCN's robustness. Lastly, RobustGCN \cite{zhu2019robust} introduces the use of Gaussian distributions as hidden representations of nodes in each convolutional layer, enabling the absorption of the effects of both structural and feature-based adversarial attacks.
From another perspective, given the limitations of the aforementioned methods in terms of theoretical guarantees, there has been a growing interest in exploring robustness certificates \cite{robust_training_zugner, bojchevski_2019} as a promising direction to quantify a model's robustness and providing attack-independent guarantees. For instance, \cite{bojchevski_certificate_2020} introduced the use of randomized smoothing techniques to offer highly scalable model-agnostic certificates for graphs. Their approach provides a robustness guarantee that is independent of the attack method employed. Furthermore, \cite{certificate_jin_2020} proposed robustness certificates specifically for GCN-based graph classification in the presence of topological perturbations. These certificates consider both local and global budgets, enabling a comprehensive robustness analysis of the model's robustness.

Recently, defending against adversarial attacks through the injection of noise into the architecture has emerged as a promising approach in the field of Computer Vision. Several studies \cite{pinot2019theoretical, liu2018robust, rakin2018parametric} have shown that noise injection can enhance the robustness of networks against adversarial perturbations. From another perspective, and in the context of GNNs, the work \cite{wu2023adversarial} investigated the effect of injecting noise, specifically adversarial weight perturbation, on improving the generalization of models. The findings of this study demonstrate that these perturbations effectively mitigate the vanishing-gradient issue and lead to significant enhancements in generalization performance. Our work extends these insights by considering the application of noise injection schemes to enhance the robustness of GNNs against adversarial attacks.

\section{Preliminaries} \label{sec:preliminaries}
We begin by introducing several fundamental concepts.

\textbf{Notation and Problem Setup.} 
Let $G = (V,E)$ be a graph where $V$ is its set of vertices and $E$ its set of edges.
We will denote by $n = |V|$ and $m = |E|$ the number of vertices and number of edges, respectively.
Let $\mathcal{N}(v)$ denote the set of neighbors of a node $v \in V$, \ie $\mathcal{N}(v) = \{ u \colon (v,u) \in E\}$.
The degree of a node is equal to its number of neighbors, \ie equal to $|\mathcal{N}(v)|$ for a node $v \in V$.
A graph is commonly represented by its adjacency matrix $A \in \mathbb{R}^{n \times n}$ which encodes edge information.
The $(i,j)$-th element of the adjacency matrix is equal to the weight of the edge between the $i$-th and $j$-th node of the graph and a weight of $0$ in case the edge does not exist.
In some settings, the nodes of a graph might be annotated with feature vectors.
We use $X \in \mathbb{R}^{n \times K}$ to denote the node features where $K$ is the feature dimensionality.
The feature of the $i$-th node of the graph corresponds to the $i$-th row of $X$.
In a node classification setting, we consider a graph $G$, represented by its adjacency matrix $A$ and its node attribute matrix $X$. Formally, given a set of labeled $V_L \subset V$, where nodes are assigned exactly one class in $\mathcal{C} = \{ y_1, y_2, \ldots, y_c \} \subset \mathcal{Y}$, the goal is to learn a function $f_{\theta}$, which maps each node $v \in V$ to exactly one of the $c$ classes in $\mathcal{C}$ while minimizing a classification loss (the cross entropy loss for example).

\textbf{GNNs.}
A GNN model consists of a series of neighborhood aggregation layers that use the graph structure and the node feature vectors from the previous layer to generate new representations for the nodes.
Specifically, GNNs update node feature vectors by aggregating local neighborhood information.
Suppose we have a GNN model that contains $T$ neighborhood aggregation layers.
Let also $\mathbf{h}_v^{(0)}$ denote the initial feature vector of node $v$, \ie the row of matrix $X$ that corresponds to node $v$.
At each iteration ($t > 0$), the hidden state $\mathbf{h}_v^{(t)}$ of a node $v$ is updated as follows:
\begin{equation*}
    \begin{split}
        \mathbf{a}_v^{(t)} &= \textrm{AGGREGATE}^{(t)} \Big( \big\{ \mathbf{h}_u^{(t-1)} \colon u \in \mathcal{N}(v) \big\} \Big); \\
        \mathbf{h}_v^{(t)} &= \textrm{COMBINE}^{(t)} \Big(\mathbf{h}_v^{(t-1)}, \mathbf{a}_v^{(t)} \Big),
    \end{split}
\end{equation*}
where $\textrm{AGGREGATE}$ is a permutation invariant function that maps the feature vectors of the neighbors of a node $v$ to an aggregated vector.
This aggregated vector is passed along with the previous representation of $v$, \ie $\mathbf{h}_v^{(t-1)}$, to the $\textrm{COMBINE}$ function which combines those two vectors and produces the new representation of $v$.

\section{Proposed Approach}
\label{sec:proposed_approach}
In this section, we provide a mathematical formalization of robustness specifically tailored to Graph Neural Networks (GNNs). Subsequently, we investigate the impact of noise injection on the robustness of GNNs. Throughout our analysis, without loss of generality, we will focus on the semi-supervised node classification task as a representative scenario. 
Let us consider the following three metric spaces, the graph space associated with the adjacency matrices ($\mathcal{A}, \lVert  \cdot \rVert_{\mathcal{A}}$), the feature space associated with the node feature attributes ($\mathcal{X}$, $\lVert  \cdot \rVert_{\mathcal{X}}$) and the label space ($\mathcal{Y}, \lVert  \cdot \rVert_{\mathcal{Y}}$). We further consider an underlying probability distribution $\mathcal{D}$ defined on $(\mathcal{A}, \mathcal{X})$. Throughout this section, $\lVert \cdot \lVert$ denotes the Euclidean (resp., spectral) norm for vectors (resp., matrices).

\subsection{Graph Adversarial Attacks}

Let us consider a trained victim classifier $f: (\mathcal{A}, \mathcal{X}) \rightarrow \mathcal{Y}$ and let $(A, X) \in (\mathcal{A}, \mathcal{X})$ be an input graph with its associated label vectors $y \in \mathcal{Y}$, such that $f(A, X) = y$. The objective of an adversarial attack is to generate a perturbed graph, represented by its adjacency matrix $\tilde{A}$ and the corresponding features $\tilde{X}$, which is slightly different from the original input $(A, X)$, and whose prediction is different from the original one. The adversarial aim can be therefore formulated as the search for a perturbed attributed graph $(\tilde{A}, \tilde{X})$ within a defined similarity budget $\epsilon$, such that $f(\tilde{A}, \tilde{X}) \neq f(A,X)$. 
From this perspective, we can define the adversarial risk of a GNN as the expected behavior or output of adjacent graphs to a given input graph's neighborhood within a budget $\epsilon$. This can be mathematically formulated as the following:
\begin{equation}\label{equation:robustness_definition}
\mathcal{R}_{\epsilon}[f] = \mathop{\mathbb{E}}_{\substack{(A, X) \sim \mathcal{D} \\ (\tilde{A}, \tilde{X}) \in \mathcal{N}_{\epsilon}(A, X)}} [d_{\mathcal{Y}}(f(\tilde{A}, \tilde{X}), f(A, X))],
\end{equation}
with $\mathcal{N}_{\epsilon}(A, X) = \{(\tilde{A}, \tilde{X}):  d_{\mathcal{A}, \mathcal{X}}((A, X), (\tilde{A}, \tilde{X}))<\epsilon\}$ being the input's graph neighborhood containing the valid adversarial candidates for any budget $\epsilon \geq 0$ and $d_{\mathcal{A},\mathcal{X}}$ and $d_{\mathcal{Y}}$ can be any defined distances in the measurable input and output, spaces $(\mathcal{A},\mathcal{X})$ and $\mathcal{Y}$. In our analysis, we will consider a distance metric that takes into account both the graph structure and its associated node features.
\begin{equation*} \label{eqn_GraphDistanceMetric}
d_{\mathcal{A}, \mathcal{X}}((A, X), (\tilde{A}, \tilde{X})) = \min_{\substack{P \in \Pi}} \{ \lVert A - P \tilde{A} P^T \rVert_{2} + \lVert X - P \tilde{X} \rVert_2 \},
\end{equation*}
with $\Pi$ being the set of permutation matrices. While we will be focusing on the $\ell_2$ norm, other norms may be used depending on the relevant penalization constraints corresponding to the considered use case.

Quantifying the precise adversarial risk of a GNN, as defined in Equation (\ref{equation:robustness_definition}), poses a significant challenge. However, an effective and more manageable approach is to establish an upper bound on this risk. By deriving such an upper bound, users can gain a comprehensive understanding of the GNN's susceptibility to adversarial attacks and make informed assessments of its robustness based on the specific task at hand. For example, in certain scenarios like social networks, where a limited number of successful attacks may not have severe consequences, a larger upper bound on the adversarial risk might be tolerable. While in other more sensitive areas, such as financial applications, we need to aim for a much tighter upper bound to control the confidence level of the adversarial risk. From this perspective, we introduce the notion of a GNN's robustness as follows.

\begin{definition}
\label{def:robustness}
(Adversarial Robustness). The graph-based function $f: (\mathcal{A}, \mathcal{X}) \rightarrow \mathcal{Y}$ is said to be $(\epsilon, \gamma)-\text{robust}$ if its adversarial risk is upper-bounded, \ie $\mathcal{R}_{\epsilon}[f] \leq \gamma$ with respect to the chosen graph distances in the input and output metric spaces.
\end{definition}

Our introduced robustness formulation deviates from the common approach seen in the literature, which typically focuses on evaluating worst-case scenarios using specific adversarial examples. Instead of considering the model's behavior only under individual adversarial instances, our method examines how the model performs more generally within a defined neighborhood. This perspective leans towards a concept of ``average'' robustness, expanding on the conventional worst-case based adversarial robustness that is often emphasized in adversarial studies. We argue that our ``average'' definition provides a more comprehensive grasp of the model's robustness, encompassing the classical quantification of adversarial vulnerabilities. In fact, ensuring ``average'' robustness inherently guarantees ``worst-case'' robustness as will be demonstrated in Proposition \ref{pro:generalization_robustness}: 
\begin{proposition}
\label{pro:generalization_robustness}
    Let $f: (\mathcal{A}, \mathcal{X}) \rightarrow \mathcal{Y}$ be a graph-based function, with respect to the chosen input and output space distance, the following holds: 
    
    $f$ is $(\epsilon, \gamma)-\text{robust}$ $\Rightarrow$ $f$ is $(\epsilon, \gamma)-\text{``worst-case'' robust}$. 
\end{proposition}
As a result of Proposition \ref{pro:generalization_robustness}, our upper-bound analysis can be used for both our ``average'' robustness measure and the traditional ``worst-case'' adversarial robustness, effectively connecting the two. The detailed proof of Proposition \ref{pro:generalization_robustness} is in Appendix A. 


\subsection{Effect of Noise Injection}

Our study aims to investigate the effect of noise addition in term of defending against adversarial attacks. 
Specifically, we consider injecting noise sampled from a predefined distribution during the training and inference time to enhance the robustness of an underlying GNN. We should hence, and similar to \citet{pinot2019theoretical}, consider a probabilistic space as our output space and our victim model $f$ as a probabilistic mapping where an output is obtained by sampling from the mapping. Accordingly, we will consider the Kullback–Leibler (KL) divergence as the corresponding output distance $d_\mathcal{Y}$ as defined in our adversarial risk quantification introduced in Equation~(\ref{equation:robustness_definition}).

Our analysis will focus on the widely used GCN and GIN within the broader context of GNNs. For illustration, and as introduced in the ``Preliminaries'' Section, we can write an iteration of the iterative process of GCN as follows:
\begin{equation} \label{equation:gcn}
    H^{(\ell)} = \phi^{(\ell)}(\hat{A}H^{(\ell-1)}W^{(\ell)}),
\end{equation}
where $H^{(\ell)}$ represents the hidden state in the $\ell$-th GCN layer with $H^{(0)}$ corresponding to the initial node features $X\in~\mathbb{R}^{n \times K}$, $W^{(\ell)} \in \mathbb{R}^{p \times e}$ is the weight matrix in the $\ell$-th layer, $e$ is the embedding dimension and $\phi^{(\ell)}$ is a non-linear activation function. Moreover, $\hat{A}~\in~\mathbb{R}^{n \times n}$ is the normalized adjacency matrix  $\hat{A} = D^{-1/2} A D^{-1/2}$. We note that for the GIN, we use the following adaptation of the adjacency matrix $\hat{A} = A + (1 + \lambda) I$.  

In the remainder of our theoretical analysis, we consider our victim model to be any GCN or GIN based graph classifier. We additionally assume that $f$ contains only 1-Lipschitz continuous activation functions, which is the case for commonly used activation functions such as the Hyperbolic Tangent. While in practice one can choose to to sample the injected noise from a variety of distributions, our theoretical study will focus on the centered Gaussian distribution $\mathcal{N}(0, I)$ with a scaling parameter $\beta$ controlling its covariance matrix.
The work \cite{wu2023adversarial}, which mainly focused on connecting adversarial weight perturbation to generalization, has shown that injecting noise at each layer can lead to a collapse in the model's generalization. As a result, we will restrict the introduction of noise to specific layers to yield better results. Under these assumptions, our victim model can be expressed as $f(\cdot) = \Phi^{\ell} \circ \cdots \circ \Phi^{i+1} (\Phi^{i} \circ \cdots \circ \Phi^{1}(\cdot) + T)$, where $T$ represents a Gaussian random variable.

\begin{theorem}
\label{theo:main_result}
    Let $f$ denote a graph-based function composed of $2$ layers and based on 1-Lipschitz continuous activation functions. We consider injecting noise drawn from a centered Gaussian with a scaling parameter $\beta$. When subject to structural perturbations of the input graph $(A, X),$ with a budget $\epsilon$, we have with respect to Definition~\ref{def:robustness}:

    \begin{itemize}
        \item If $f$ is GCN-based then $f$ is $(\epsilon, \gamma)-\text{robust}$ with
            \begin{center}
                $\gamma = \frac{2(\lVert W^{(2)} \lVert \lVert W^{(1)} \lVert \lVert X \lVert \epsilon)^2}{\beta};$
            \end{center}
        \item If $f$ is GIN-based then $f$ is $(\epsilon, \gamma)-\text{robust}$ with
            \begin{center}
                $\gamma = \frac{(\lVert W^{(2)} \lVert \lVert W^{(1)} \lVert \lVert X \lVert \epsilon (2 \lVert A \lVert + \epsilon))^2}{2\beta},$
            \end{center}         
    \end{itemize}
where $W^{(\ell)}$ denotes the weight matrix of the $\ell$-th layer. 
\end{theorem}

Theorem \ref{theo:main_result} provides an upper bound on a GCN and GIN based graph classifier's robustness and establishes the connection between noise injection and defending against adversarial attacks based on structural perturbations with a predefined neighborhood and budget $\epsilon$. Since a tighter upper bound intuitively signifies a higher level of robustness in the targeted victim model, based on the results derived from the theorem, controlling the injected noise using the $\beta$ parameter can effectively enhance the model's robustness. However, it is important to exercise caution when increasing the injected noise as it can potentially compromise the model's performance. Hence, striking a balance between defending against adversarial attacks and preserving the model's clean accuracy becomes crucial. Furthermore, although the previous theorem focuses on a 2-layers graph classifier known for its benchmark accuracy across diverse datasets, the results can be extended to graph classifiers with $L$ layers. 
Another notable observation from the analysis of Theorem~\ref{theo:main_result} is that the computed upper bound for a GCN-based classifier is significantly tighter compared to that of a GIN-based classifier. This suggests that the GCN model exhibits greater robustness to structural perturbations, which aligns with our intuition since the normalization of the adjacency matrix is expected to attenuate the effects of adversarial perturbations. The proof of Theorem \ref{theo:main_result} is provided in Appendix~B.


While Theorem \ref{theo:main_result} and our experimental analysis focus on structural perturbations, similar analysis can be applied to node feature-based adversarial attacks. In this context, Theorem \ref{theo:result_feature_based} sheds light on the link between noise injection and enhancing robustness when the underlying model is subject to adversarial attacks targeting node features. 

\begin{theorem}
\label{theo:result_feature_based}
    Let $f$ denote a graph-based classifier composed of $2$ layers and based on 1-Lipschitz continuous activation functions. We consider injecting noise drawn from a centered Gaussian with a scaling parameter $\beta$. When subject to node feature-based perturbations of the input graph $(A, X)$, we have with respect to Definition \ref{def:robustness}:

    \begin{itemize}
        \item If $f$ is GCN-based then $f$ is $(\epsilon, \gamma)-\text{robust}$ with
            \begin{center}
                $\gamma = \frac{(\lVert W^{(2)} \lVert \lVert W^{(1)} \lVert  \epsilon)^2}{2\beta}$;
            \end{center}
        \item If $f$ is GIN-based then $f$ is $(\epsilon, \gamma)-\text{robust}$ with
            \begin{center}
                $\gamma = \frac{(\lVert A \lVert \lVert W^{(2)} \lVert \lVert W^{(1)} \lVert  \epsilon)^2}{2\beta}, $
            \end{center}         
    \end{itemize}
    
    where $W^{(\ell)}$ denotes the weight matrix of the $\ell$-th layer. 
\end{theorem}
The proof of Theorem \ref{theo:result_feature_based} can be found in Appendix C.


\subsection{Complexity and Advantage of Our Approach}

Many existing defense methods suffer from a significant increase in complexity as the input graph size grows, making them challenging to apply in practical scenarios. For example, GNNGuard \cite{gnn_guard} involves computing neighbor importance estimation, which has a complexity of $\mathcal{O}(e \times |E|)$, where $e$ denotes the embedding dimension and $|E|$ represents the number of edges. GCN-Jaccard, which preprocesses the network by eliminating edges connecting nodes with a Jaccard similarity of features smaller than a chosen threshold, has a complexity of $\mathcal{O}(|E|)$. Moreover, GNN-SVD, to discard the high-rank perturbations, computes a low-rank approximation of the adjacency and features matrices derived from their SVD for which the complexity is $\mathcal{O}(|V|^3)$.
In contrast, our proposed approach based on noise injection in the architecture is advantageous due to its minimal complexity, requiring only sampling from a distribution. In this perspective, we note that we provide an experimental comparison analysis of the training time of the different cited methods against our proposed framework in Appendix E. 
Additionally, unlike many existing methods, our approach does not compromise the performance of the underlying GCN when applied to clean, non-attacked graphs, as will be shown in our experimental results.

\def\arraystretch{1.3}

\begin{table*}[h]

\centering
\resizebox{\textwidth}{!}{%
\fontsize{22pt}{22pt}\selectfont
\begin{tabular}{cllccccc|ccccc}
\hline
\multirow{2}{*}{}   & \multirow{2}{*}{Dataset}                      & \multicolumn{1}{c}{\multirow{2}{*}{$\epsilon$}} & \multicolumn{5}{c|}{GCN}                                                                                                        & \multicolumn{5}{c}{GIN}                                                                                                         \\ \cline{4-13} 
                          &                                               & \multicolumn{1}{c}{}                               & Guard                & Jaccard             & SVD                 & RGNN                    & Noisy                & Guard                & Jaccard             & SVD                 & RGNN                    & Noisy                \\ \hline
\multirow{12}{*}{\rotatebox{90}{Mettack}} & \multirow{3}{*}{Cora}                         & 0\%                                              & 77.5$\pm$0.7          & 80.9$\pm$0.7          & 80.6$\pm$0.4          & \textbf{83.5$\pm$0.3} & 83.2$\pm$0.4          & 82.5$\pm$0.3          & 80.8$\pm$0.5          & 79.7$\pm$1.0          & \textbf{83.5$\pm$0.3} & 83.0$\pm$0.2          \\
                          &                                               & 5\%                                       & 75.8$\pm$0.6          & 78.9$\pm$0.8          & 78.4$\pm$0.6          & 78.3$\pm$0.6          & \textbf{81.2$\pm$0.7} & 79.5$\pm$0.7          & 78.9$\pm$0.8          & 78.2$\pm$1.3          & 78.3$\pm$0.6          & \textbf{81.1$\pm$0.5} \\
                          &                                               & 10\%                                      & 74.7$\pm$0.4          & \textbf{76.7$\pm$0.7} & 71.5$\pm$0.8          & 70.7$\pm$0.8          & 74.5$\pm$0.6          & 73.8$\pm$1.3          & 76.4$\pm$0.5          & 73.9$\pm$1.2          & 70.7$\pm$0.8          & \textbf{76.7$\pm$0.8} \\ \cline{2-13} 
                          & \multirow{3}{*}{CiteSeer}                     & 0\%                                              & 70.1$\pm$1.5          & 71.2$\pm$0.7          & 70.7$\pm$0.4          & \textbf{72.3$\pm$0.5} & 71.9$\pm$0.4          & 71.2$\pm$0.6          & \textbf{72.5$\pm$1.2} & 71.9$\pm$1.6          & 72.3$\pm$0.5          & 71.9$\pm$0.6          \\
                          &                                               & 5\%                                       & 69.9$\pm$1.1          & 70.3$\pm$2.3          & 68.9$\pm$0.7          & 70.6$\pm$0.7          & \textbf{72.3$\pm$0.6} & \textbf{71.8$\pm$0.7} & 70.9$\pm$1.7          & 70.8$\pm$1.8          & 70.6$\pm$0.7          & 71.3$\pm$0.8          \\
                          &                                               & 10\%                                      & 70.0$\pm$1.5          & 67.5$\pm$2.1          & 68.8$\pm$0.6          & 68.7$\pm$1.2          & \textbf{70.4$\pm$0.8} & 67.8$\pm$1.1          & \textbf{70.2$\pm$1.5} & 69.3$\pm$1.8          & 68.7$\pm$1.2          & 69.2$\pm$1.3          \\ \cline{2-13} 
                          & \multirow{3}{*}{PubMed}                       & 0\%                                              & 84.5$\pm$0.6          & 85.0$\pm$0.5          & 82.7$\pm$0.3          & \textbf{85.1$\pm$0.8} & 85.0$\pm$0.6          & \textbf{85.1$\pm$0.6} & 84.9$\pm$0.9          & 82.8$\pm$0.3          & \textbf{85.1$\pm$0.8} & 84.8$\pm$0.4          \\
                          &                                               & 5\%                                       & \textbf{84.3$\pm$0.9} & 79.6$\pm$0.3          & 81.3$\pm$0.6          & 81.1$\pm$0.7          & 81.8$\pm$0.4          & \textbf{83.2$\pm$0.5} & 81.6$\pm$0.7          & 82.1$\pm$0.7          & 81.1$\pm$0.7          & 82.4$\pm$0.9          \\
                          &                                               & 10\%                                      & \textbf{84.1$\pm$0.3} & 67.4$\pm$1.1          & 81.1$\pm$0.7          & 65.2$\pm$0.4          & 73.3$\pm$0.6          & 78.5$\pm$0.9          & 77.5$\pm$1.5          & \textbf{81.6$\pm$0.6} & 65.2$\pm$0.4          & 78.9$\pm$1.8          \\ \cline{2-13} 
                          & \multirow{3}{*}{PolBlogs}                     & 0\%                                              & 93.1$\pm$0.6          & -                       & 86.5$\pm$0.8          & 94.9$\pm$0.3          & \textbf{95.2$\pm$0.4} & \textbf{95.6$\pm$0.9} & -                       & 93.4$\pm$0.6          & 95.2$\pm$0.3          & 94.9$\pm$0.7          \\
                          &                                               & 5\%                                       & 72.8$\pm$0.8          & -                       & \textbf{85.1$\pm$1.6} & 76.0$\pm$0.8          & 79.7$\pm$0.6          & 94.5$\pm$0.8          & -                       & 92.8$\pm$0.9          & 76.0$\pm$0.8          & \textbf{94.7$\pm$0.5} \\
                          &                                               & 10\%                                      & 68.7$\pm$1.0          & -                       & \textbf{84.8$\pm$2.3} & 69.2$\pm$1.2          & 73.4$\pm$0.5          & 92.5$\pm$0.9          & -                       & 92.1$\pm$1.6          & 69.2$\pm$1.2          & \textbf{92.8$\pm$0.6} \\ \hline
\multirow{8}{*}{\rotatebox{90}{PGD}}      & \multirow{2}{*}{Cora}                         & 5\%                                       & 71.0$\pm$1.0          & 73.9$\pm$0.8          & 69.9$\pm$0.6          & 75.8$\pm$0.9          & \textbf{76.6$\pm$0.3} & \textbf{81.8$\pm$1.1} & 80.1$\pm$0.6          & 74.6$\pm$1.2          & 75.8$\pm$0.9          & 81.3$\pm$0.4          \\
                          &                                               & 10\%                                      & 69.9$\pm$1.6          & 72.2$\pm$1.4          & 65.3$\pm$0.9          & 72.4$\pm$1.8          & \textbf{73.4$\pm$0.5} & \textbf{81.0$\pm$1.5} & 79.7$\pm$0.8          & 73.9$\pm$1.1          & 72.4$\pm$1.8          & 79.9$\pm$0.7          \\ \cline{2-13} 
                          & \multirow{2}{*}{CiteSeer}                     & 5\%                                       & 57.9$\pm$2.8          & 62.9$\pm$1.5          & 61.7$\pm$1.3          & 58.1$\pm$2.2          & \textbf{64.5$\pm$1.2} & 69.8$\pm$0.4          & 70.1$\pm$0.5          & 67.9$\pm$0.9          & 58.1$\pm$2.2          & \textbf{70.7$\pm$0.5} \\
                          &                                               & 10\%                                      & 58.2$\pm$3.8          & 61.3$\pm$0.7          & 59.5$\pm$0.3          & 56.2$\pm$0.8          & \textbf{62.2$\pm$1.0} & 68.9$\pm$0.9          & 69.4$\pm$0.7          & 65.6$\pm$1.3          & 56.2$\pm$0.8          & \textbf{70.0$\pm$0.7} \\ \cline{2-13} 
                          & \multirow{2}{*}{PubMed}                       & 5\%                                       & 75.3$\pm$0.4          & 76.1$\pm$0.7          & 67.7$\pm$1.5          & \textbf{78.5$\pm$0.8} & 76.2$\pm$0.7          & \textbf{81.0$\pm$0.3} & 80.8$\pm$0.6          & 80.8$\pm$0.9          & 78.5$\pm$0.8          & 80.6$\pm$0.4          \\
                          &                                               & 10\%                                      & \textbf{70.7$\pm$0.9} & 64.7$\pm$1.2          & 67.5$\pm$1.7          & 65.6$\pm$0.9          & 65.2$\pm$1.1          & \textbf{80.3$\pm$0.7} & 79.9$\pm$0.8          & 80.1$\pm$1.2          & 65.6$\pm$0.9          & 79.6$\pm$0.6          \\ \cline{2-13} 
                          & \multicolumn{1}{l}{\multirow{2}{*}{PolBlogs}} & 5\%                                       & 76.8$\pm$0.6          & -                       & 82.1$\pm$1.1          & 82.5$\pm$0.4          & \textbf{83.2$\pm$0.7} & 93.4$\pm$0.4          & -                       & 88.4$\pm$1.2          & 82.5$\pm$0.4          & \textbf{94.0$\pm$0.6} \\
                          & \multicolumn{1}{l}{}                          & 10\%                                      & 74.3$\pm$0.8          & -                       & \textbf{80.2$\pm$1.5} & 76.5$\pm$0.7          & 77.6$\pm$0.9          & 91.5$\pm$1.2          & -                       & 86.2$\pm$1.6          & 76.5$\pm$0.7          & \textbf{92.1$\pm$0.7} \\ \hline
\multirow{8}{*}{\rotatebox{90}{DICE}}     & \multirow{2}{*}{Cora}                         & 5\%                                       & 76.4$\pm$0.4          & 79.6$\pm$0.6          & 74.9$\pm$1.3          & 81.9$\pm$0.6          & \textbf{82.5$\pm$0.8} & 81.9$\pm$0.3          & 79.7$\pm$0.8          & 79.3$\pm$0.8          & \textbf{81.9$\pm$0.6} & 81.7$\pm$0.5          \\
                          &                                               & 10\%                                      & 76.6$\pm$0.5          & 78.6$\pm$0.8          & 73.5$\pm$1.5          & 80.0$\pm$0.6          & \textbf{80.5$\pm$0.6} & 79.2$\pm$0.6          & 78.6$\pm$0.6          & 78.1$\pm$1.1          & 80.0$\pm$0.6          & \textbf{80.5$\pm$0.8} \\ \cline{2-13} 
                          & \multirow{2}{*}{CiteSeer}                     & 5\%                                       & 68.5$\pm$1.6          & \textbf{70.9$\pm$0.4} & 69.4$\pm$1.6          & 69.3$\pm$0.5          & 70.8$\pm$0.3          & 69.6$\pm$1.5          & 70.3$\pm$0.7          & 65.7$\pm$1.8          & 69.3$\pm$0.5          & \textbf{70.8$\pm$0.9} \\
                          &                                               & 10\%                                      & 69.9$\pm$1.5          & 69.9$\pm$0.6          & 68.1$\pm$1.5          & 67.8$\pm$1.1          & \textbf{70.4$\pm$0.8} & 68.3$\pm$0.7          & 69.3$\pm$0.6          & 64.5$\pm$2.3          & 67.8$\pm$1.1          & \textbf{69.6$\pm$1.2} \\ \cline{2-13} 
                          & \multirow{2}{*}{PubMed}                       & 5\%                                       & \textbf{84.0$\pm$0.8} & 83.4$\pm$0.7          & 81.5$\pm$0.8          & 83.8$\pm$0.6          & 83.6$\pm$0.9          & \textbf{84.0$\pm$0.4} & 83.5$\pm$0.3          & 82.3$\pm$0.7          & 83.8$\pm$0.6          & 83.8$\pm$0.3          \\
                          &                                               & 10\%                                      & \textbf{83.6$\pm$1.0} & 81.8$\pm$0.5          & 81.4$\pm$0.5          & 82.4$\pm$0.8          & 82.1$\pm$2.3          & \textbf{82.9$\pm$0.8} & 82.0$\pm$0.5          & 82.0$\pm$0.9          & 82.4$\pm$0.8          & 82.5$\pm$0.6          \\ \cline{2-13} 
                          & \multicolumn{1}{l}{\multirow{2}{*}{PolBlogs}} & 5\%                                       & 81.3$\pm$0.7          & -                       & 86.5$\pm$2.3          & 89.6$\pm$0.4          & \textbf{90.3$\pm$0.3} & 93.3$\pm$0.3          & -                       & 90.9$\pm$0.5          & 89.6$\pm$0.4          & \textbf{93.5$\pm$0.5} \\
                          & \multicolumn{1}{l}{}                          & 10\%                                      & 78.9$\pm$0.6          & -                       & 85.3$\pm$2.8          & 85.5$\pm$0.9          & \textbf{86.1$\pm$0.9} & \textbf{91.1$\pm$1.2} & -                       & 89.9$\pm$0.8          & 85.5$\pm$0.9          & 89.7$\pm$0.8          \\ \hline
\end{tabular}%
}

\caption{Classification accuracy ($\pm$ standard deviation) of the models on different benchmark node classification datasets for different perturbation rates $\epsilon.$  The best accuracy in each setting, each dataset, and each model is typeset in bold.}
\label{tab:results_node_classification}

\end{table*}

\section{Experimental Results}
\label{sec:experimental_results}

This section focuses on empirically validating our theoretical findings by evaluating the performance of the proposed approach on real-world benchmarks. We begin by outlining the experimental settings employed in our study, followed by a comprehensive analysis and discussion of the obtained results. Through our experimental evaluation, we aim to address three main key aspects: firstly, the effectiveness of our method in defending against adversarial attacks, particularly structural perturbations, and secondly, its capability to maintain the model's accuracy and performance, especially when tested on non-attacked input graphs and finally its complexity in terms of training time compared to other methods.

\subsection{Experimental Setup}
We focus on node classification where we use the citation networks Cora, CiteSeer, and PubMed \cite{dataset_node_classification} and blog and citation graphs, \ie Polblogs \cite{polblogs_dataset} and OGBN-Arxiv \cite{hu2021open}. Note that in the Polblogs graph, node features are not available. Further information about the used datasets and implementation details are provided in Appendix~D. 
For all the experiments, the baseline models consist of a 2-layer GCN-based classifier combined with an MLP as a readout. This choice aimed to ensure a fair evaluation of the models' robustness within the same architectural conditions. The experiments were conducted using the Adam optimizer \cite{kingma_adam} and standardized hyperparameters, including a learning rate of 1e-2, 300 epochs, and 16 as the hidden dimension. 
To reduce the impact of random initialization, we repeated each experiment 10 times and used the train/validation/test splits provided with the datasets \cite{yang_2016}. Our code is publicly available in GitHub\footnote{Code: https://github.com/Sennadir/NoisyGNN}.


\textbf{Attacks.} We use three main global structural-based adversarial attacks: \textbf{(i)} We first consider the optimization-based formulation of the adversarial task Mettack with the “Meta-Self” training strategy. \textbf{(ii)} We afterward consider another optimization-based adversarial attack based on Proximal Gradient Descent (PGD) \cite{pgd_paper} and \textbf{(iii)} we finally consider DICE \cite{zugner2019adversarial}. For all these attacks, we tested and considered two perturbation budgets (in term of percentage) $\epsilon \in \{5\%,10\%\}$. 


\textbf{Baseline Models.} To provide a comprehensive empirical evaluation, we compared our proposed defense algorithm, NoisyGCN, against four baseline methods that specifically address structural perturbations. The baseline models we consider are:
\textbf{(i)} GNN-Jaccard \citep{gnn_jaccard} that preprocesses the input adjacency matrix to identify potential edge manipulations. \textbf{(ii)} RGCN \citep{zhu2019robust}, that employs Gaussian distributions as hidden representation to absorb the effect of structural adversarial attacks. \textbf{(iii)} GNN-SVD, which employs a low-rank approximation of the adjacency matrix to filter out noise, and \textbf{(iv)} GNNGuard \citep{gnn_guard}, which is based on edge pruning to defend against adversarial perturbations.

\subsection{Experimental Results}

Table~\ref{tab:results_node_classification} presents the average node classification accuracies for the GCN, the GNNGuard, GNN-Jaccard, RGNN, GNN-SVD, and the proposed approach, NoisyGNN for both GCN and GIN. 
The empirical findings reveal that, in the absence of attacks, the proposed approach demonstrates comparable accuracy to the classical GCN, and in some cases, it even improves the model's generalization and performance, as studied and analyzed by prior research \cite{wu2023adversarial}. Importantly, these results affirm that our approach does not compromise the performance of the underlying network, addressing our second research question. This is particularly significant as real-world scenarios often involve uncertain knowledge regarding potential malicious perturbations on the input graph. Hence, it is crucial that an effective defense strategy does not diminish the predictive capabilities of the model, while simultaneously enhancing its robustness.
The results furthermore indicate that our proposed noise injection approach performs on par with and even surpasses state-of-the-art defense baselines in several instances when working with both GCN and GIN. Notably, it demonstrates greater efficiency when subjected to the “PGD” and “DICE” attack framework. Moreover, we observe an almost consistent outperformance compared to GNN-SVD, GNN-Jaccard, and RGNN, while also exhibiting competitive performance against the highly performant GNNGuard. It is important to highlight that despite similar performance to GNNGuard, the main advantage of our approach is its theoretical guarantees but also its significantly reduced complexity in terms of operation and time. The complete time analysis study is provided in Appendix~E. 

\begin{figure}[t]
\centering
\begin{minipage}{.5\columnwidth}
  \centering
  \includegraphics[width=0.99\linewidth]{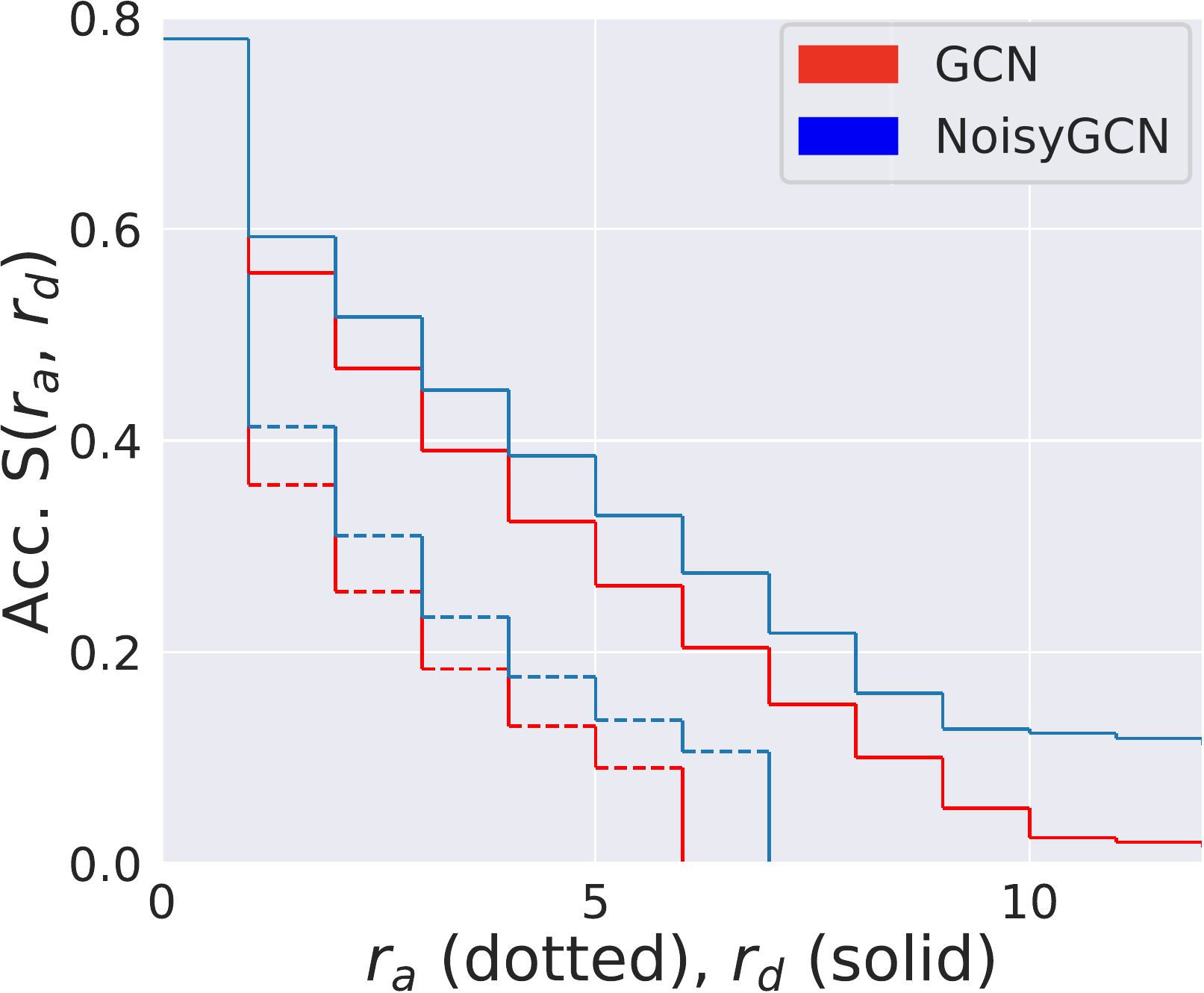}
  \caption*{(a) Cora}
  \label{fig:test1}
\end{minipage}%
\begin{minipage}{.50\columnwidth}
  \centering
  \includegraphics[width=0.99\linewidth]{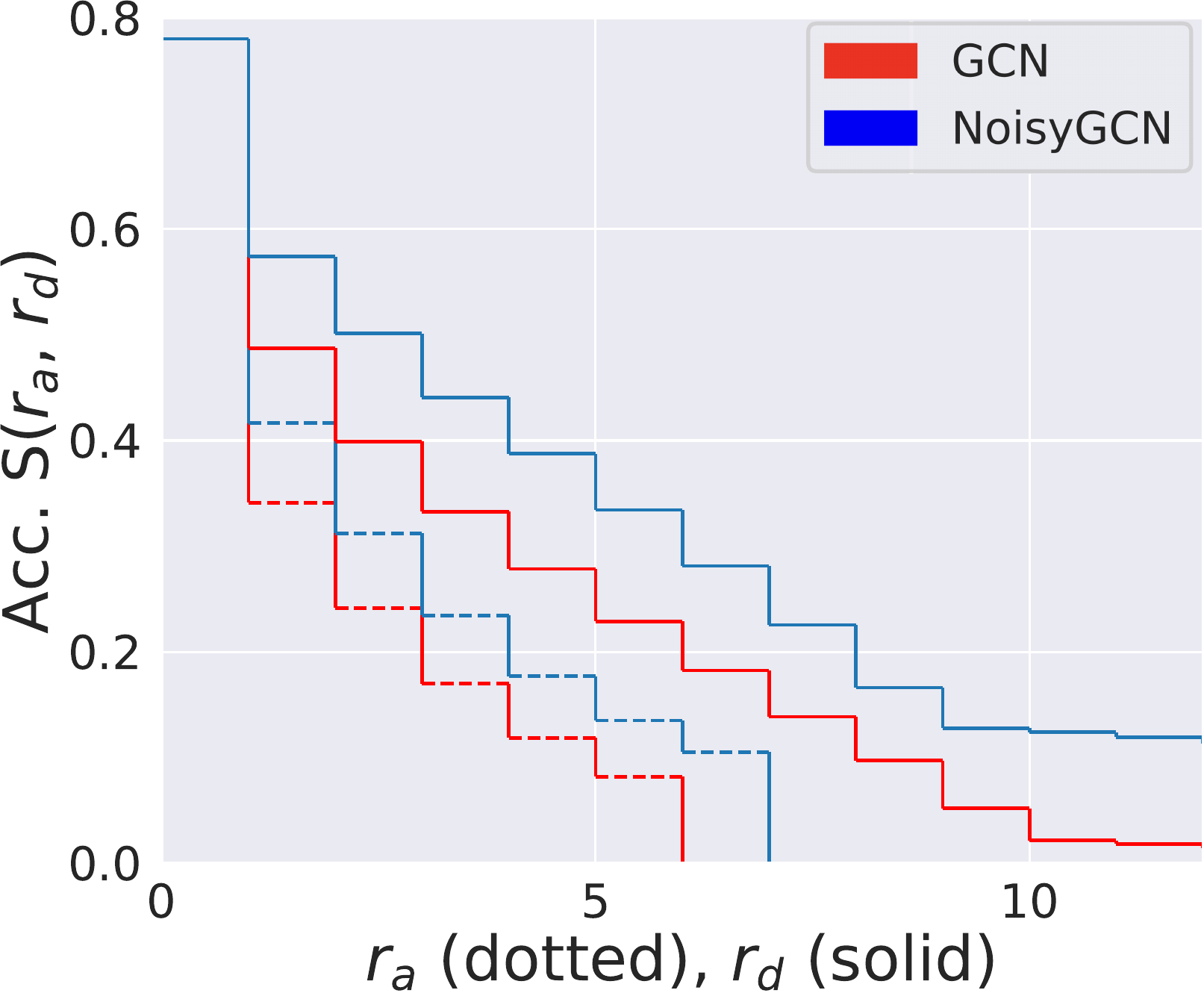}
  
 \caption*{(b) CiteSeer}
  \label{fig:test2}
\end{minipage}

\caption{Robustness guarantees on (a) Cora and (b) CiteSeer, where $r_a$ is the certified radius – maximum number of adversarial additions (and $r_d$ for deletions).}
 \label{fig:certified_robustness}
\end{figure}

\textbf{Evaluation Through Robustness Certificates.} 
We furthermore conducted a comparative analysis between the GCN and our proposed NoisyGCN where we assessed their certified robustness on the Cora and CiteSeer datasets using the sparse randomized smoothing approach \cite{bojchevski_certificate_2020}. Specifically, we plotted their certified accuracy $S(r_a, r_d)$ while varying the radii for addition ($r_a$) and deletion ($r_d$) of edges.
Figure~\ref{fig:certified_robustness} illustrates the results, showcasing that our proposed noise injection method significantly enhances the certified accuracy of NoisyGCN when subject to structural perturbations. This observation emphasizes the effectiveness of our approach in improving the model's robustness.

\textbf{Node Feature Based Adversarial Attacks.} 
We furthermore empirically study Theorem \ref{theo:result_feature_based} concerning node feature-based adversarial attacks. We consider two main feature-based adversarial attacks: 
\textbf{(i)} The baseline random noise attack adding Gaussian noise $\mathcal{N}(0, \mathbf{I})$ to the node features with a scaling parameter $\xi$ to control the attack budget;  
\textbf{(ii)} The  white-box Proximal Gradient Descent (PGD) \cite{pgd_paper} attack. We note that this attack is more adapted for continuous spaces and hence is known to be powerful in the node feature space. For this analysis, we focused on the citation network between Computer Science arXiv papers OGBN-Arxiv \cite{hu2021open} to demonstrate the ability of our method on larger datasets. We note that this dataset could not be used for the structural perturbations since the majority of the available attacks require access to a dense form of the adjacency matrix which is not very realistic at large scale.
We compare our method against defense methods adapted for feature-based adversarial perturbations: \textbf{(i)} GCN-k \cite{seddik_rgcn} that proposes to enhance the robustness of GNNs to noise and adversarial attacks by incorporating a node feature kernel into the message passing operators; \textbf{(ii)} RobustGCN (RGCN) leveraging Gaussian distributions as hidden representations to absorb the impact of adversarial attacks; \textbf{(iii)} AIRGNN \cite{airgnn} that edited the Message Passing module with adaptive residual connections and feature aggregation to improve the GNN's robustness against abnormal node features. 
Table \ref{tab:node_features_attacks} reports the clean and attacked accuracy for the considered benchmark defense methods alongside our introduced NoisyGCN. The findings highlight our method's ability to defend against node feature-based adversarial attacks, thereby affirming the theoretical conclusions drawn in Theorem \ref{theo:result_feature_based}.

\begin{table}[t]

\begin{center}
\resizebox{\columnwidth}{!}{%
\fontsize{16pt}{16pt}\selectfont
\begin{tabular}{l|cccc}
\toprule
Attack & GCN-k & AirGNN & RGCN & NoisyGCN \\
\midrule
Clean    & 56.0$\pm$0.3 & 61.9$\pm$0.9 & 65.1$\pm$1.8 & \textbf{66.9$\pm$0.5} \\

\midrule

$\xi=0.5$    & 52.8$\pm$0.5 & 59.0$\pm$1.3 & 63.8$\pm$1.9 & \textbf{65.8$\pm$2.0} \\
$\xi=1.0$   & 46.6$\pm$0.6 & 51.9$\pm$1.6 & 63.0$\pm$2.4 & \textbf{64.3$\pm$1.3} \\

\midrule

PGD    & 49.9$\pm$0.7 & 55.7$\pm$0.9 & 63.6$\pm$0.7 & \textbf{64.9$\pm$1.1} \\

\bottomrule
\end{tabular} 
}

\caption{Classification accuracy ($\pm$ standard deviation) of the models on the OGBN-Arxiv node classification dataset before (``Clean'') and after the attack application.}

\label{tab:node_features_attacks}

\end{center}
\end{table}

\textbf{Combining Adversarial Defenses.} 
\label{sec:combine_defenses}
As previously discussed, our approach is model-agnostic, allowing it to be applied to various GNN architectures. Therefore, it is meaningful to experimentally assess the impact of combining our proposed noise injection with other benchmark defense methods, especially for large attack budgets. In this context, we consider methods that do not modify the model architecture but instead introduce a defense layer. Specifically, we examine pre-processing methods such as GNN-Jaccard \cite{gnn_jaccard} and edge pruning methods like GNNGuard \cite{gnn_guard}.
However, certain methods such as RGNN \cite{zhu2019robust} are excluded from this section due to their major edits in the architecture, which do not directly align with our proposed noise injection. We report the results for both GCN and GIN for the Mettack with an attack budget of $\epsilon=25\%$ while we provide results of the other attacks and budgets in Appendix F. 

The results in Table \ref{tab:combine_results} demonstrate the improved defense results obtained by combining our noise injection approach with existing defense methods across various datasets and underlying models. Notably, the combination shows promising outcomes in mitigating the added defense layer's effect on the clean accuracy, addressing a significant weakness in methods like GNNGuard. These findings underscore the capability of our method to enhance the robustness of any underlying architecture with minimal additional complexity.

\def\arraystretch{0.9}

\begin{table}[t]


\begin{center}
\begin{tabular}{l|ccc}
\toprule
Method & Cora & CiteSeer & PolBlogs \\
\midrule
GINGuard    & 61.8$\pm$0.5 & 55.6$\pm$1.8 & 82.7$\pm$0.6 \\
+ Noisy & \textbf{66.2$\pm$1.3} & \textbf{58.3$\pm$1.9} & \textbf{83.6$\pm$0.8} \\
\midrule
GIN-Jaccard    & 70.4$\pm$1.1 & 61.2$\pm$2.3 & - \\
+ Noisy & \textbf{72.9$\pm$0.8} & \textbf{64.9$\pm$1.8 }& - \\
\midrule

GCNGuard    & 69.5$\pm$0.7 & 66.2$\pm$0.6 & 64.7$\pm$0.8 \\
+ Noisy & \textbf{72.4$\pm$1.2} & \textbf{68.9$\pm$0.9} & \textbf{65.8$\pm$1.3}\\
\midrule

GCN-Jaccard    & 66.7$\pm$0.5 & 61.2$\pm$1.1 & - \\
+ Noisy & \textbf{69.6$\pm$0.9} &\textbf{63.1$\pm$0.6} & -\\

\bottomrule
\end{tabular} 

\caption{Classification accuracy ($\pm$ standard deviation) of combining defense methods with the proposed noise injection on different benchmark datasets. }

\label{tab:combine_results}

\end{center}
\end{table}


\section{Conclusion}
In this study, we present NoisyGNN, a highly effective and cost-effective defense method for Graph Neural Networks (GNNs). Through rigorous theoretical analysis, we establish a clear and compelling connection between noise injection in the victim model's architecture and enhanced robustness. Our proposed method offers a significant advantage by introducing minimal additional complexity while delivering strong defense performance.
Comprehensive experimental comparisons conducted on diverse real-world datasets demonstrate that our proposed framework achieves comparable or even superior performance when compared to standard GCN and GIN models, as well as existing defense methods specifically designed for those models. This highlights the effectiveness and versatility of our approach.
While our primary focus in this study was on the GCN and GIN architectures, it is important to note that our approach is model-agnostic, as demonstrated by successful combinations of our defense method with other existing techniques. Additionally, our theoretical analysis has the potential to be explored and extended to other GNN architectures, thus opening up avenues for further research and application.

\section*{Acknowledgements}
This work was partially supported by the Wallenberg AI, Autonomous Systems and Software Program (WASP) funded by the Knut and Alice Wallenberg Foundation. The computation (on GPUs) was enabled by resources provided by the National Academic Infrastructure for Supercomputing in Sweden (NAISS) at Alvis partially funded by the Swedish Research Council through grant agreement no. ``2022-06725''. This work was granted access to the HPC resources of IDRIS under the allocation ``2023-AD010613410R1'' made by GENCI. Y.A. is supported by the French National research agency via the AML-HELAS (ANR-19-CHIA-0020) project.

\bibliography{aaai24}


\newpage
\appendix
\onecolumn

\vbox{%
  \hsize\textwidth%
  \linewidth\hsize%
  \vskip 0.625in minus 0.125in%
  \centering%
  {\LARGE\bf Appendix: Proofs And Experimental Details \par}%
  \vskip 0.1in plus 0.5fil minus 0.05in%
}%


\section{Proof of Proposition \ref{pro:generalization_robustness}}\label{sec:proof_proposition}

\setcounter{theorem}{1}

\begin{proposition}
Let $f: (\mathcal{A}, \mathcal{X}) \rightarrow \mathcal{Y}$ be a graph-based function, with respect to the chosen input and output space distance, the following holds: 
    \begin{center}
    $f$ is $(\epsilon, \gamma)-\text{robust}$ $\Rightarrow$ $f$ is $(\epsilon, \gamma)-\text{``worst-case'' robust}$.
    \end{center}
\end{proposition}

\begin{proof}

Let $(A,X)$ be the considered input graph and let $\mathcal{N}_{\epsilon}(A, X) = \{(\tilde{A}, \tilde{X}):  d_{\mathcal{A}, \mathcal{X}}((A, X), (\tilde{A}, \tilde{X}))<\epsilon\}$ be its corresponding neighborhood for an $\epsilon$ attack budget. Let $\mathcal{S}_{\epsilon}(A, X)$ be the set of worst-case adversarial examples within the considered budget $\epsilon$ and let's denote $\mathcal{R}^w_{\epsilon}[f]$ as the adversarial risk related to worst-case evaluation defined as:

\begin{equation*}
    \mathcal{R}^w_{\epsilon}[f] = \mathop{\mathbb{E}}_{\substack{(A, X) \sim \mathcal{D} \\ (\tilde{A}, \tilde{X}) \in \mathcal{S}_{\epsilon}(A, X)}} [d_{\mathcal{Y}}(f(\tilde{A}, \tilde{X}), f(A, X))].
\end{equation*}

By definition, we have:  $\mathcal{S}_{\epsilon}(A, X) \subseteq \mathcal{N}_{\epsilon}(A, X)$ and hence: 
\begin{equation}
\label{equation:inequality_risk}
    \mathcal{R}^w_{\epsilon}[f] \leq \mathcal{R}_{\epsilon}[f].
\end{equation}

Let us consider a graph-based function $f$ to be $(\epsilon, \gamma)-\text{robust},$ \ie $\mathcal{R}_{\epsilon}[f] \leq \gamma$. Based on Equation \eqnref{equation:inequality_risk}, we also have that $\mathcal{R}^w_{\epsilon}[f] \leq \gamma$ and consequently have the stated result, \ie $f$ is also $(\epsilon, \gamma)-\text{``worst-case'' robust}$. 

\end{proof}

\section{Proof of Theorem \ref{theo:main_result}}
\label{sec:proof_theorem}

\begin{theorem}
    Let $f$ denote a graph-based function composed of $2$ layers and based on 1-Lipschitz continuous activation functions. We consider injecting noise drawn from a centered Gaussian with a scaling parameter $\beta$. When subject to structural perturbations of the input graph $(A, X),$ with a budget $\epsilon$, we have with respect to Definition~\ref{def:robustness}:

    \begin{itemize}
        \item If $f$ is GCN-based then $f$ is $(\epsilon, \gamma)-\text{robust}$ with
                $$\gamma = \frac{2(\lVert W^{(2)} \lVert \lVert W^{(1)} \lVert \lVert X \lVert \epsilon)^2}{\beta};$$
        \item If $f$ is GIN-based then $f$ is $(\epsilon, \gamma)-\text{robust}$ with
                $$\gamma = \frac{(\lVert W^{(2)} \lVert \lVert W^{(1)} \lVert \lVert X \lVert \epsilon (2 \lVert A \lVert + \epsilon))^2}{2\beta},$$
    \end{itemize}
where $W^{(\ell)}$ denotes the weight matrix of the $\ell$-th layer. 

\end{theorem}

\begin{proof}
We begin by recalling the notation and definition of the involved quantities. 
We consider $f$ to be a 2-layer GCN-based classifier with 1-Lipschitz continuous activation functions. Let $A'$ denote the perturbed adjacency matrix produced as a result of the considered structural perturbations.  While, we consider injecting noise that is sampled from a Gaussian distribution $\mathcal{N}(0, I)$ with a scaling parameter $\beta$, we start by considering the general case of a centered Gaussian with a variance parameter $\Sigma$ (with the maximal singular value of its inverse denoted by $\sigma_{max}(\Sigma^{-1})$). Further recall that the embedding dimension of the hidden states in our graph-based classifier is denoted by $e.$ If we let $\ast$ denote the convolution operation then the distribution of our predictions $f(A,X)$ is $p_G \ast \delta_{f(A, X)},$ where $p_G$ denotes the Gaussian probability density function arising as a consequence of the Gaussian noise we add to the hidden states, and the Dirac delta, $\delta_{f(A, X)},$ corresponding to the trivial distribution of the prediction in the absence of noise.

Let us consider the general case of the Renyi Divergence denoted by $d_{R, \lambda}$ and $\mu$ denoting the corresponding probability measure over which we integrate.  Note that the KL divergence can be recovered from the  Renyi Divergence when $\lambda \rightarrow 1.$ 
\begin{align*}
   d_{R, \lambda}(f(A, X), f(A', X))  & = \frac{1}{\lambda - 1} \log \int_{\mathbb{R}^e} (p_G \ast \delta_{f(A, X)})^\lambda (p_G \ast \delta_{f(A', X)})^{(1 - \lambda)} d\mu\\
                    & =  \frac{1}{\lambda - 1} \log \int_{\mathbb{R}^e} \frac{e^{-1/2[\lambda(z - f(A, X))^T \Sigma^{-1} (z-f(A,X)) + (1 - \lambda)(z - f(A', X))^T \Sigma^{-1} (z-f(A',X)) ] }}{ (2\pi)^{\frac{e}{2}} \vert\Sigma\vert^{\frac{1}{2}}}dz \\
                    & = \frac{\lambda - 1}{2} [f(A,X) - f(A', X)]^T (\Sigma^{-1}) [f(A,X) - f(A', X)] \numberthis \label{eqn:BigStep}\\
                    & \leq \frac{\lambda}{2} \sigma_{max}(\Sigma^{-1}) \lVert f(A,X) - f(A', X) \rVert^2,
\end{align*}
where Equation \eqnref{eqn:BigStep} follows by collecting all terms that do not depend on $z,$ pulling them out of the integral and then realizing that the integral is over the whole domain of a Gaussian probability function and hence equals $1.$

From this result, we can deduce the result for our output distance which is the KL divergence:
\begin{align*}
    d_{\mathcal{Y}}(f(A, X), f(A', X)) & = \lim_{\lambda \rightarrow 1} d_{R, \lambda}(f(A, X), f(A', X)) \\
                        & \leq \frac{1}{2} \sigma_{max}(\Sigma^{-1}) \lVert f(A,X) - f(A', X) \lVert^2. \numberthis  \label{eq:kl_result}
\end{align*}

In what follows, we study the effect of input perturbation which is reflected by the quantity  $\lVert f(A,X) - f(A', X) \lVert$ in the previous result.

In our considered 2-layers GCN $f$, we denote the activation functions in the first and second layers by $\phi^{(1)}$ and $\phi^{(2)},$ respectively. We further denote the operations performed by the first GCN layer, taking the adjacency matrix and node features as input, by $\Phi^{(1)} (A, X).$ The normalized adjacency matrix that corresponds to the GCN message passing operator is denoted by $\Tilde{A'}$ for adjacency matrix $A'$ and $\Tilde{A}$ for $A$. Note that the spectral norm of both $\Tilde{A}$ and $\Tilde{A'}$ equals to one, i.e., $\lVert \Tilde{A'} \rVert_2 = \lVert \Tilde{A} \rVert_2=1.$ 
We can then express the normed difference of GCN predictions as follows.

\begin{align*}
   \lVert f(A, X) - f(A', X) \lVert  & = \lVert \phi^{(2)}(\Tilde{A} \Phi^{(1)} (A, X)  W^{(2)})                                               - \phi^{(2)}(\Tilde{A'} \Phi^{(1)} (A', X)  W^{(2)}) \lVert \\
                    & \leq \lVert \Tilde{A} \Phi^{(1)} (A, X)  W^{(2)} - \Tilde{A'} \Phi^{(1)} (A', X)  W^{(2)}  \lVert \\
                    & \leq \lVert W^{(2)} \lVert \lVert \Tilde{A} \Phi^{(1)} (A, X) - \Tilde{A} \Phi^{(1)} (A', X) + \Tilde{A} \Phi^{(1)} (A', X)  - \Tilde{A'} \Phi^{(1)} (A', X) \lVert \\
                    & \leq \lVert W^{(2)} \lVert \lVert \Tilde{A} [ \Phi^{(1)} (A, X) - \Phi^{(1)} (A', X)] + \Phi^{(1)} (A', X) [\Tilde{A} - \Tilde{A'}] \lVert. 
\end{align*}

We also have the following.
\begin{align*}
   \lVert \Phi^{(1)} (A, X) - \Phi^{(1)} (A', X) \lVert  & = \lVert \phi^{(1)}(\Tilde{A} X  W^{(1)})                                               - \phi^{(1)}(\Tilde{A'} X  W^{(1)}) \lVert \\
                    & \leq \lVert \Tilde{A} X  W^{(1)} - \Tilde{A'} X  W^{(1)}  \lVert \\
                    & \leq \lVert W^{(1)} \lVert \lVert X \lVert \epsilon.
\end{align*}

From the previous result and using the triangular inequality, we have
\begin{align*}
   \lVert W^{(2)} \lVert \lVert \Tilde{A} [ \Phi^{(1)} (A, X) - \Phi^{(1)} (A', X)] + \Phi^{(1)} (A', X) [\Tilde{A} - \Tilde{A'}] \lVert  & \leq  \lVert W^{(2)} \lVert [ \lVert W^{(1)} \lVert \lVert X \lVert \epsilon + \lVert W^{(1)} \lVert \lVert X \lVert \epsilon    ] \\
                    & =  2 \lVert W^{(2)} \lVert \lVert W^{(1)} \lVert \lVert X \lVert \epsilon.
\end{align*}

We finally conclude from Equation \eqnref{eq:kl_result} that:
\begin{align*}
    d_{\mathcal{Y}}(f(A, X), f(A', X)) & = \frac{1}{2\beta} \lVert f(A,X) - f(A', X) \lVert^2 \\
                        & = \frac{2(\lVert W^{(2)} \lVert \lVert W^{(1)} \lVert \lVert X \lVert \epsilon)^2}{\beta}. 
\end{align*}

Let us consider $f$ to be a 2-layers GIN with its parameter $\alpha$ (which is usually referred to as $\epsilon$ in the literature), we have the following.
\begin{align*}
   \lVert f(A, X), f(A', X) \lVert  & = \lVert \phi^{(2)}(A_1 \Phi^{(1)} (A, X)  W^{(2)})                                               - \phi^{(2)}(A_2 \Phi^{(1)} (A', X)  W^{(2)}) \lVert \\
                    & \leq \lVert A_1 \Phi^{(1)} (A, X)  W^{(2)} - A_2 \Phi^{(1)} (A', X)  W^{(2)}  \lVert \\
                    & \leq \lVert W^{(2)} \lVert \lVert A_1 \Phi^{(1)} (A, X) - A_1 \Phi^{(1)} (A', X) + A_1 \Phi^{(1)} (A', X)  - A_2 \Phi^{(1)} (A', X) \lVert \\
                    & \leq \lVert W^{(2)} \lVert \lVert A_1 [ \Phi^{(1)} (A, X) - \Phi^{(1)} (A', X)] + \Phi^{(1)} (A', X) [A_1 - A_2] \lVert,
\end{align*}

with $A_1 = A + (1 + \alpha) I$ and $A_2 = A' + (1 + \alpha) I.$

Similar to the previous, we have 
\begin{align*}
   \lVert \Phi^{(1)} (A, X) - \Phi^{(1)} (A', X) \lVert  & = \lVert \phi^{(1)}(A_1 X  W^{(1)})                                               - \phi^{(1)}(A_2 X  W^{(1)}) \lVert \\
                    & \leq \lVert A_1 X  W^{(1)} - A_2 X  W^{(1)}  \lVert \\
                    & \leq \lVert W^{(1)} \lVert \lVert X \lVert \lVert A + (1 + \alpha) I - A' - (1 + \alpha) I \lVert \\
                    & \leq \lVert W^{(1)} \lVert \lVert X \lVert \epsilon.
\end{align*}

Hence, we have
\begin{align*}
   \lVert W^{(2)} \lVert \lVert A_1 [ \Phi^{(1)} (A, X) - \Phi^{(1)} (A', X)] + \Phi^{(1)} (A', X) [A_1 - A_2] \lVert  & \leq  \lVert W^{(2)} \lVert [ \lVert A_1 \lVert \lVert W^{(1)} \lVert \lVert X \lVert \epsilon + \lVert A_2 \lVert \lVert W^{(1)} \lVert \lVert X \lVert \epsilon    ] \\
                    & =  \lVert W^{(2)} \lVert \lVert W^{(1)} \lVert \lVert X \lVert \epsilon ( 2 \lVert A_1 \lVert + \epsilon) .
\end{align*}

We finally conclude from Equation \eqnref{eq:kl_result} that:
\begin{align*}
    d(f(A, X), f(A', X)) & = \frac{1}{2\beta} \lVert f(A,X) - f(A', X) \lVert^2 \\
                        & = \frac{(\lVert W^{(2)} \lVert \lVert W^{(1)} \lVert \lVert X \lVert \epsilon (2 \lVert A_1 \lVert + \epsilon))^2}{2\beta}.
\end{align*}

\end{proof}

\section{Proof of Theorem \ref{theo:result_feature_based}} \label{app:ProofFeaturePerturbations}

\begin{theorem}
    Let $f$ denote a graph-based classifier composed of $2$ layers and based on 1-Lipschitz continuous activation functions. We consider injecting noise drawn from a centered Gaussian with a scaling parameter $\beta$. When subject to node feature-based perturbations of the input graph $(A, X)$, we have with respect to Definition \ref{def:robustness}:

    \begin{itemize}
        \item If $f$ is GCN-based then $f$ is $(\epsilon, \gamma)-\text{robust}$ with
            \begin{center}
                $\gamma = \frac{(\lVert W^{(2)} \lVert \lVert W^{(1)} \lVert  \epsilon)^2}{2\beta}$;
            \end{center}
        \item If $f$ is GIN-based then $f$ is $(\epsilon, \gamma)-\text{robust}$ with
            \begin{center}
                $\gamma = \frac{(\lVert A \lVert \lVert W^{(2)} \lVert \lVert W^{(1)} \lVert  \epsilon)^2}{2\beta}, $
            \end{center}         
    \end{itemize}
    
    where $W^{(\ell)}$ denotes the weight matrix of the $\ell$-th layer.

\end{theorem}

\begin{proof}
In this part, we consider node feature-based adversarial attack, we denote the produced perturbed node features matrix as $X'$.
Using the same analogy as in the proof to Theorem \ref{theo:main_result}, we can directly deduce 

\begin{align}
    d_{\mathcal{Y}}(f(A, X), f(A, X')) & = \lim_{\lambda \rightarrow 1} d_{R, \lambda}(f(A, X), f(A, X')) \\
                        & = \frac{1}{2} \sigma_{max}(\Sigma^{-1}) \lVert f(A,X) - f(A, X') \lVert^2. \label{eq:new_kl_result}
\end{align}

Let us consider that $f$ is a 2-layer GCN, we can write the following.
\begin{align*}
   \lVert f(A, X) - f(A, X') \lVert  & = \lVert \phi^{(2)}(\Tilde{A} \Phi^{(1)} (A, X)  W^{(2)})                                               - \phi^{(2)}(\Tilde{A} \Phi^{(1)} (A, X')  W^{(2)}) \lVert \\
                    & \leq \lVert \Tilde{A} \lVert \lVert W^{(2)} \rVert \lVert \Phi^{(1)} (A, X)  - \Phi^{(1)} (A, X')  \lVert.
\end{align*}

By recurrence, we directly get

\begin{align*}
   \lVert f(A, X) - f(A, X') \lVert  & = \lVert \phi^{(2)}(\Tilde{A} \Phi^{(1)} (A, X)  W^{(2)})                                               - \phi^{(2)}(\Tilde{A} \Phi^{(1)} (A, X')  W^{(2)}) \lVert \\
                    & \leq \lVert W^{(2)} \lVert \lVert W^{(1)}\rVert  \lVert X  -  X'  \lVert \\
                    & \leq \lVert W^{(2)} \lVert \lVert W^{(1)} \lVert \epsilon.
\end{align*}

We finally conclude from Equatino \eqnref{eq:new_kl_result} that
\begin{align*}
    d_{\mathcal{Y}}(f(A, X), f(A, X')) & = \frac{1}{2\beta} \lVert f(A,X) - f(A, X') \lVert^2 \\
                        & = \frac{(\lVert W^{(2)} \lVert \lVert W^{(1)} \lVert  \epsilon)^2}{2\beta} 
\end{align*}

Let us consider $f$ to be a 2-Layers GIN, let $A_1 = A + (1 + \lambda)I$ we can write the following
\begin{align*}
   \lVert f(A, X) - f(A, X') \lVert  & = \lVert \phi^{(2)}(A_1 \Phi^{(1)} (A, X)  W^{(2)})                                               - \phi^{(2)}(A_1 \Phi^{(1)} (A, X')  W^{(2)}) \lVert \\
                    & \leq \lVert A_1 \lVert \lVert W^{(2)} \rVert  \lVert \Phi^{(1)} (A, X)  - \Phi^{(1)} (A, X')  \lVert.
\end{align*}

Similar to the previous case, recursively we have
\begin{align*}
   \lVert f(A, X)- f(A, X') \lVert  & = \lVert \phi^{(2)}(A_1 \Phi^{(1)} (A, X)  W^{(2)})                                               - \phi^{(2)}(A_1 \Phi^{(1)} (A, X')  W^{(2)}) \lVert \\
                    & \leq \lVert A_1 \lVert \lVert W^{(2)} \lVert \lVert W^{(1)} \lVert X  -  X'  \lVert \\
                    & \leq \lVert A_1 \lVert \lVert W^{(2)} \lVert \lVert W^{(1)} \lVert \epsilon.
\end{align*}

So we can conclude from Equation \eqnref{eq:new_kl_result} that:
\begin{align*}
    d(f(A, X) - f(A, X')) & = \frac{1}{2\beta} \lVert f(A,X) - f(A, X') \lVert^2 \\
                        & = \frac{(\lVert A_1 \lVert \lVert W^{(2)} \lVert \lVert W^{(1)} \lVert  \epsilon)^2}{2\beta}. 
\end{align*}
\end{proof}


\section{Datasets and Implementation Details} \label{appendix:additional_informations}

Characteristics and information about the used datasets in the node classification part of the study are presented in Table \ref{tab:data_statistics}. As outlined in the main paper, we conduct experiments on the citation networks Cora, CiteSeer, PubMed \cite{dataset_node_classification} and PolBlogs \cite{polblogs_dataset}. We additionally consider OGBN-Arxiv \cite{hu2021open} dataset for the node features-based adversarial attacks. For these benchmarks, we use the train/valid/test splits provided with the datasets.

\begin{table}[h]
\caption{Statistics of the node classification datasets used in our experiments.}
\label{tab:data_statistics}
\vskip 0.15in
\begin{center}
\begin{small}
\begin{sc}
\begin{tabular}{lcccc}
\toprule
Dataset & \#Features & \#Nodes & \#Edges & \#Classes \\
\midrule
Cora    & 1433 & 2708 & 5208 & 7 \\
CiteSeer    & 3703 & 3327 & 4552 & 6 \\
PubMed    & 500 & 19717 & 44338 & 3 \\
PolBlogs    & N/A & 1222 & 16717 & 2 \\
OGBN-Arxiv     & 128 & 31971 & 71669 & 40 \\
\bottomrule
\end{tabular}
\end{sc}
\end{small}
\end{center}
\vskip -0.1in
\end{table}

\subsection{On the Used Hyper-parameters}

For all the experiments, the used models and benchmarks consist of a 2-layer convolutional architecture. This structure involves two rounds of message passing and updating, coupled with a Multi-Layer Perception (MLP) serving as a readout. The aim was to ensure a fair comparison level for assessing the models' robustness.
We employed the Adam optimizer \cite{kingma_adam} and kept hyperparameters constant across experiments. These shared settings consisted of a learning rate of 1e-2, a training phase that consisted of 300 epochs, and a hidden dimension of 16. Worth noting, that for the OGB dataset, we adjusted the hidden dimension to 512 to suit dataset-specific considerations and to achieve state-of-the-art initial accuracy.
To account for the impact of random initialization, we repeated each experiment 10 times. This allowed us to calculate the mean and standard deviation of the results, providing a more comprehensive analysis. 
For our noise injection method, we considered the noise scaling factor $\beta$ to be a hyperparameter and we accordingly tuned this parameter using the validation data.
Lastly, a specific adaptation was made for the AIRGNN model where we set the parameter $K$ to 2 to align the number of propagations with the other benchmark models.

\subsection{Implementation Details}

Our implementation is built using the open-source library \textit{PyTorch Geometric} (PyG) under the MIT license \cite{Fey/Lenssen/2019}. We leveraged the publicly available implementation of the different benchmarks from their available repositories~:~From GNNGuard\footnote{https://github.com/mims-harvard/GNNGuard}, GCN-k\footnote{https://github.com/ChangminWu/RobustGCN} and AIRGNN\footnote{https://github.com/lxiaorui/AirGNN}. For RGCN, GCN-Jaccard and GCN-SVD, we used the implementation from the DeepRobust package \cite{li2020deeprobust}. Note that we additionally utilized the PyTorch DeepRobust package\footnote{https://github.com/DSE-MSU/DeepRobust} to implement the adversarial attacks used in this study. The experiments have been run on a Tesla V100 GPU.

\section{Time Complexity Analysis} \label{appendix:time_analysis}
Our proposed NoisyGCN, based on noise injection, offers a clear advantage in terms of its required complexity. It only requires sampling from a pre-defined distribution during each forward pass, independent of the graph size. Therefore, its complexity does not increase with the graph size, unlike other baselines. For example, GNNGuard, the state-of-the-art defense method, involves computing the neighbor's importance estimation, which results in complexity that scales with the input graph. We empirically validate this observation by comparing the training time complexity of each considered defense method. 

\begin{table}[h]
\small
\caption{Mean training time analysis (in s) of the NoisyGNN in comparison to other baselines for both the GCN and GIN instances.}
\label{tab:time_analysis}
\vskip 0.01in
\begin{center}
\begin{small}
\begin{sc}
\begin{tabular}{lccccc}
\toprule
Dataset & GCNGuard & GCN-Jaccard & RGCN & GCN-SVD & NoisyGCN  \\
\midrule
Cora    & 28.52 & 1.93 & 1.16 & 1.39 & 1.29   \\
CiteSeer    & 36.04 & 1.58 & 1.23 & 1.12 & 1.24  \\
PubMed    & 731.26 & 12.27 & 34.19 & 4.60 & 2.41  \\
PolBlogs    & 18.17 & 5.17 & 0.96 & 0.80 & 0.65  \\

\bottomrule
\bottomrule

Dataset & GINGuard & GIN-Jaccard & RGCN & GIN-SVD & NoisyGIN  \\
\midrule
Cora    & 48.93 & 3.12 & 1.31 & 1.51 & 1.93  \\
CiteSeer    & 58.45 & 3.78 & 1.44 & 2.20 & 2.76  \\
PubMed    & 963.58 & 16.28 & 41.09 & 6.33 & 7.86  \\
PolBlogs    & 43.7 & 5.52 & 0.95 & 3.71 & 3.16  \\
\bottomrule

\end{tabular}
\end{sc}
\end{small}
\end{center}
\vskip -0.1in
\end{table}

The analysis of training time, as presented in Table \ref{tab:time_analysis}, highlights the distinct time complexities observed between NoisyGCN and the other baseline methods. Specifically, there is a notable disparity in training time complexity between NoisyGCN and GNNGuard. While GCN-SVD exhibits a comparable time complexity to our approach, the superior defense capabilities of NoisyGCN, as demonstrated in Table \ref{tab:results_node_classification}, differentiate it from GCN-SVD. Furthermore, the results obtained on the PubMed dataset affirm the motivation outlined in our paper, illustrating that the majority of existing methods impose a complexity burden when dealing with large graphs.

\section{Additional Experiments - Combining Defense and Noise Injection} \label{appendix:additional_analysis_combinition}

Following the same course as outlined in Section \ref{sec:combine_defenses}, we conducted empirical evaluations to gauge the impact of integrating our noise injection technique with established benchmark defense methods, particularly in scenarios involving substantial attack budgets. To this end, we combined our proposed method with GNN-Jaccard \cite{gnn_jaccard} and GNNGuard \cite{gnn_guard}. Just as in our principal assessment illustrated in Table \ref{tab:results_node_classification}, we subjected the combined defenses to adversarial attacks including Mettack, PGD, and DICE, for an attack budget of $\epsilon = 25\%$.

\begin{table}[h]
\centering

\caption{Classification accuracy ($\pm$ standard deviation) of combining defense methods with the proposed noise injection on different benchmark datasets. The higher the accuracy (in \%) the better the model.}

\resizebox{\columnwidth}{!}{%
\begin{tabular}{lccc|ccc|ccc}
\cline{2-10}
                  & \multicolumn{3}{c|}{Mettack}                     & \multicolumn{3}{c|}{PGD}                         & \multicolumn{3}{c}{DICE}                         \\ \cline{2-10} 
                  & Cora           & CiteSeer       & PolBlogs       & Cora           & CiteSeer       & PolBlogs       & Cora           & CiteSeer       & PolBlogs       \\ \hline
GCNGuard          & 69.5 $\pm$ 0.7 & 66.2 $\pm$ 0.6 & 64.7 $\pm$ 0.8 & 71.2 $\pm$ 1.3 & 68.4 $\pm$ 0.8 & 67.6 $\pm$ 1.4 & 74.6 $\pm$ 0.2 & 67.3 $\pm$ 0.9 & 69.8 $\pm$ 1.1 \\
+ Noise Injection & 72.4 $\pm$ 1.2 & 68.9 $\pm$ 0.9 & 65.8 $\pm$ 1.3 & 72.9 $\pm$ 1.6 & 69.8 $\pm$ 0.6 & 68.1 $\pm$ 1.2 & 75.2 $\pm$ 0.7 & 68.9 $\pm$ 0.6 & 70.6 $\pm$ 1.4 \\ \hline
GCN-Jaccard       & 66.7 $\pm$ 0.5 & 61.2 $\pm$ 1.1 & -              & 67.5 $\pm$ 0.8 & 62.9 $\pm$ 1.8 & -              & 74.6 $\pm$ 0.2 & 76.3 $\pm$ 0.8 & -              \\
+ Noise Injection & 69.6 $\pm$ 0.9 & 63.1 $\pm$ 0.6 & -              & 69.1 $\pm$ 1.1 & 63.8 $\pm$ 1.6 & -              & 75.2 $\pm$ 0.7 & 76.3 $\pm$ 0.8 & -              \\ \hline
GINGuard          & 61.8 $\pm$ 0.5 & 55.6 $\pm$ 1.8 & 82.7 $\pm$ 0.6 & 77.4 $\pm$ 1.2 & 67.0 $\pm$ 1.1 & 87.2 $\pm$ 2.9 & 76.3 $\pm$ 0.8 & 65.8 $\pm$ 1.5 & 87.3 $\pm$ 0.7 \\
+ Noise Injection & 66.2 $\pm$ 1.3 & 58.3 $\pm$ 1.9 & 83.6 $\pm$ 0.8 & 79.1 $\pm$ 0.6 & 66.2 $\pm$ 0.6 & 88.0 $\pm$ 2.6 & 77.5 $\pm$ 0.7 & 66.4 $\pm$ 1.2 & 87.1 $\pm$ 0.4 \\ \hline
GIN-Jaccard       & 70.4 $\pm$ 1.1 & 61.2 $\pm$ 2.3 & -              & 78.3 $\pm$ 0.6 & 67.1 $\pm$ 0.7 & -              & 76.2 $\pm$ 0.6 & 65.9 $\pm$ 0.8 & -              \\
+ Noise Injection & 72.9 $\pm$ 0.8 & 64.9 $\pm$ 1.8 & -              & 79.0 $\pm$ 1.2 & 67.5 $\pm$ 0.4 & -              & 76.7 $\pm$ 0.8 & 66.8 $\pm$ 1.1 & -              \\ \hline
\end{tabular}%
}
\end{table}


\end{document}